\documentclass[12pt]{article}

\title{Best-of-Three-Worlds Linear Bandit Algorithm\\with Variance-Adaptive Regret Bounds}

\usepackage{fullpage}
\usepackage{times}
\usepackage{algorithm}
\usepackage{algorithmic}
\usepackage{booktabs}
\usepackage{graphicx}
\usepackage{svg}
\usepackage{amsmath}
\usepackage{amssymb}
\usepackage{natbib}
\usepackage{amsthm}
\usepackage{hyperref}
\newtheorem{theorem}{Theorem}
\newtheorem{lemma}{Lemma}

\theoremstyle{definition}
\newtheorem{definition}{Definition}
\newtheorem{remark}{Remark}

\newcommand{\conv}{\mathrm{conv}}

\newcommand{\cE}{\mathcal{E}}

\newcommand{\cA}{\mathcal{A}}
\newcommand{\cX}{\mathcal{X}}

\newcommand{\cD}{\mathcal{D}}

\newcommand{\Rs}{\mathcal{R}^{\mathrm{sto}}}
\DeclareMathOperator*{\argmin}{arg\,min}

\DeclareMathOperator*{\E}{\mathbf{E}}

\newcommand{\re}{\mathbb{R}}

\newcommand{\linner}{\left\langle}
\newcommand{\rinner}{\right\rangle}
\renewcommand{\theta}{\vartheta}
\renewcommand{\epsilon}{v}

\author{%
Shinji Ito\thanks{NEC Corporation. Email:
\href{mailto:i-shinji@nec.com}{\texttt{i-shinji@nec.com}},
\href{mailto:kei\_takemura@nec.com}{\texttt{kei\_takemura@nec.com}}.
}
\and
Kei Takemura$^*$
}

\begin{document}
\date{}
\maketitle

\begin{abstract}%
	This paper proposes a linear bandit algorithm that is adaptive to environments at two different levels of hierarchy.  At the higher level, the proposed algorithm adapts to a variety of types of environments.  More precisely, it achieves best-of-three-worlds regret bounds, i.e., of ${O}(\sqrt{T \log T})$ for adversarial environments and of $O(\frac{\log T}{\Delta_{\min}} + \sqrt{\frac{C \log T}{\Delta_{\min}}})$ for stochastic environments with adversarial corruptions, where $T$, $\Delta_{\min}$, and $C$ denote, respectively, the time horizon, the minimum sub-optimality gap, and the total amount of the corruption.  Note that polynomial factors in the dimensionality are omitted here.  At the lower level, in each of the adversarial and stochastic regimes, the proposed algorithm adapts to certain environmental characteristics, thereby performing better.  The proposed algorithm has data-dependent regret bounds that depend on all of the cumulative loss for the optimal action, the total quadratic variation, and the path-length of the loss vector sequence.  In addition, for stochastic environments, the proposed algorithm has a variance-adaptive regret bound of $O(\frac{\sigma^2 \log T}{\Delta_{\min}})$ as well, where $\sigma^2$ denotes the maximum variance of the feedback loss.  The proposed algorithm is based on the \texttt{SCRiBLe} algorithm.  By incorporating into this a new technique we call \textit{scaled-up sampling}, we obtain high-level adaptability, and by incorporating the technique of optimistic online learning, we obtain low-level adaptability.
\end{abstract}

\section{Introduction}
This paper considers linear bandit problems.
In this class of problems,
a player chooses,
in each round $t$,
an action $a_t$ from a given \textit{action set} $\cA$,
which is a subset of a $d$-dimensional linear space.
The player then observes the incurred loss $f_t(a_t) \in [-1, 1]$,
where the (conditional) expectation of $f_t$ is assumed to be a linear function,
i.e.,
$f_t$ is expressed as $ f_t(a) = \linner \ell_t, a \rinner + \varepsilon_t (a)$ with some vector $\ell_t \in \re^d$
and some noise $\varepsilon_t$.
The performance of the player is evaluated in terms of of \textit{regret} $R_T$ defined as
$
	R_T( a^* )
	=
	\E \left[
		\sum_{t=1}^T f_t(a_t)
		-
		\sum_{t=1}^T f_t(a^*)
	\right]$
	and
$
R_T =
\max_{a^* \in \cA} R_T(a^*) 
$.

Algorithms for linear bandit problems have been proposed mainly for two different types of environments:
\textit{stochastic} and \textit{adversarial}.
In stochastic environments,
$\{ f_t \}$ are assumed to follow an unknown distribution $\cD$ independently for all $t$.
Consequently,
we may assume that
there exists $\ell^* \in \re^d$ such that $\ell_t = \ell^*$ and
$\varepsilon_t(a)$ follows an identical distribution for all $t$.\footnote{
	In standard stochastic settings,
	it is assumed that $\varepsilon_t(a)$ follows a zero-mean distribution for all $a \in \cA$.
	Our proposed algorithm,
	however,
	works well under milder assumptions,
	details of which are given in Section \ref{sec:setup} and Remark \ref{rem:sto}.
}
In adversarial environments,
the distributions of $f_t$ (and thus also $\ell_t$) are decided arbitrarily depending on the action sequence $(a_s)_{s=1}^{t-1}$ that the player has chosen so far.

What we can do in linear bandit problems varies greatly depending on the type of environment.
For stochastic environments,
it is known that the optimal regret is of $\Theta ( \log T )$ \citep{lattimore2017end},
ignoring the factor dependent on $d, \cA$ and $\ell^*$.
For adversarial environments,
the mini-max optimal regret is $\tilde{\Theta}( d \sqrt{T})$ \citep{bubeck2012towards},
where we ignore poly-logarithmic factors with respect to $d$ and $T$
in the notation of $\tilde{O}, \tilde{\Omega}$ and $\tilde{\Theta}$.
A class of intermediate settings between these two types of environments are called
stochastic environments with adversarial corruption \citep{lykouris2018stochastic},
or corrupted stochastic environments.
Environments in this regime are parametrized by corruption level $C$,
which measures the amount of the adversarial component.
For this setting,
an algorithm achieving $O((\log T)^2 + C)$-regret has been proposed \citep{lee2021achieving}.

The aim of this paper is to make possible the construction of \textit{adaptive} algorithms that
automatically exploit certain specific characteristics of environments.
In existing studies of bandit algorithms,
the concept of adaptability has been considered at two different levels,
regarding which we here refer to high-level adaptability and low-level adaptability.
Algorithms with high-level adaptability are designed to work well for different types of environments,
e.g.,
stochastic and adversarial types.
Algorithms with low-level adaptability perform better in specific individual environments
by exploiting certain favorable characteristics that they possess,
e.g.,
small cumulative loss or small variance in loss sequences.

High-level-adaptive bandit algorithms that perform (nearly) optimally for both stochastic and adversarial environments
are called best-of-both-worlds (BOBW) algorithms \citep{bubeck2012best}.
Among such algorithms,
those that can adapt to corrupted stochastic environments
are referred to as best-of-all-worlds \citep{erez2021towards} or best-of-three-worlds (BOTW) algorithms \citep{lee2021achieving}.
For linear bandit problems,
\citet{lee2021achieving} provide a best-of-three-worlds algorithm
that achieves regret bounds of $O((\log T)^2)$ for stochastic environments,
of $\tilde{O}(\sqrt{T})$ for adversarial environments,
and of ${O}((\log T)^2 + C)$ for corrupted stochastic environments.

Various types of low-level-adaptive algorithms have been considered for adversarial bandit problems.
Representative examples are algorithms with $\tilde{O}(\sqrt{L^*})$-regret,
where $L^*$ represents the cumulative loss for the optimal action;
such examples are said to have \textit{first-order} regret bounds.
In addition to such algorithms,
\citet{hazan2011better} proposed an algorithm with a \textit{second-order} regret bound of $\tilde{O}(\sqrt{Q})$ that depends on the total quadratic variation $Q$ of loss vectors.
An algorithm by \citet{ito2021hybrid} achieves $\tilde{O}(\sqrt{ \min\{ L^*, Q, P \} })$-regret,
which means that the algorithm simultaneously has first-order and second-order bounds
as well as a bound depending on the path-length $P$ of the loss sequence.
These regret bounds,
which are referred to as \textit{data-dependent regret bounds},
imply that algorithm performance can be improved by exploiting certain environmental characteristics that are common in applications,
such as small variations in loss sequences or sparsity of loss.
For the stochastic multi-armed bandit problem,
\citet{audibert2007tuning} proposed an algorithm with an $O( \sum_{i} (\frac{\sigma_i^2}{\Delta_i}+1) \log T)$-regret bound that depended
not only on the sub-optimality gap $\Delta_i$ but also on the variance $\sigma_i^2$ of the loss.
We refer to such bounds as \textit{variance-adaptive} bounds,
and they can be considered to represent low-level adaptability in stochastic regimes.

\subsection{Contribution of this work}
The main contribution of this paper is the proposal of a linear bandit algorithm that combines
high-level adaptability and low-level adaptability.
It is a BOTW algorithm
that achieves regret bounds of $O( \log T )$ in stochastic environments,
$\tilde{O}(\sqrt{T})$ in adversarial environments,
and $O( \log T + \sqrt{ C \log T } )$ in corrupted stochastic environments,
ignoring factors depending on $d, \cA$ and $\ell^*$.
Further,
the algorithm achieves first-order, second-order, and path-length bounds in adversarial environments.
Simultaneously,
it has variance-adaptive regret bounds for (corrupted) stochastic environments.

The proposed algorithm (Algorithm~\ref{alg:FTRL}) follows the approach of \texttt{SCRiBLe} \citep{abernethy2008competing,abernethy2012interior}
which stands for \textit{Self-Concordant Regularization in Bandit Learning}.
This approach uses a class of functions known as \textit{self-concordant barriers} \citep{nesterov1994interior}
as regularizers.
Self-concordant barriers are characterized with a parameter $\theta \geq 1$ that can be assumed to satisfy $\theta = O(d)$,
details of which are given in Section~\ref{sec:self-concordant}.
The regret bounds of our algorithm can be expressed with parameters explained in Table~\ref{table:parameters},
including the parameter $\theta$,
as follows:
\begin{theorem}[informal]
	\label{thm:main}
	In adversarial environments with $\varepsilon_t(a) = 0$,
	the regret for Algorithm~\ref{alg:FTRL} is bounded as
	$
		R_T = O \left( d \sqrt{\theta \min \{T, Q, P \} \log T} \right)
	$.
	Further,
	if $f_t(a) \geq 0$ for all $a \in \cA$ and $t \in [T]$,
	we have $
		R_T = O \left( d \sqrt{ \theta L^* \log T} \right)
	$.
	In stochastic environments (i.e., if $\ell_t = \ell^*$ for all $t$),
	we have
	$
		R_T 
		= O \left(
				(\frac{ d \sigma^2}{\Delta_{\min}} + 1)  d \theta \log T
		\right)$.
	In corrupted stochastic environments with the corruption level $C = \sum_{t=1}^T \| \ell_t - \ell^* \|_2$,
	we have $R_T 
		= O \left(
				(\frac{ d \sigma^2}{\Delta_{\min}} + 1)  d \theta \log T
			+
			\sqrt{
				C \cdot (\frac{ \sigma^2}{\Delta_{\min}} + 1)  d^2 \theta \log T
			}
		\right)$.
\end{theorem}

\begin{table}[t]
	\caption{List of parameters in regret bounds.}
	\label{table:parameters}
	\centering
	\begin{tabular}{lllc}
		\toprule
		Parameter & Description
		\\
		\midrule
		$T \in \mathbb{N}$ & time horizon
		\\
		$d \in \mathbb{N}$ & dimensionality of action set
		\\
		$\cA \subseteq \re^d$ & action set (One may assume $\log (|\cA|) = O(d \log T)$ w.l.o.g.)
		\\
		$\theta \geq 1$ & parameter of a self-concordant barrier $\psi$ used in the algorithm
		\\
		&
		(One can choose $\psi$ so that $\theta = O(d)$)
		\\
		$ \Delta_{\min} > 0 $ & minimum suboptimality gap: $\Delta_{\min} = \min_{a \in \cA \setminus \{ a^* \}} \linner \ell^*, a - a^* \rinner$
		\\
		$ c^* > 0$ & asymptotic lower bound parameter: $c^* = c(\cA, \ell^*) = O(d / \Delta_{\min})$
		\\
		$\sigma^2 \geq 0$ & maximum variance of loss: $\sigma^2 = \max_{a \in \cA, t} \E \left[ (f_t(a) - \linner \ell^*, a \rinner)^2 \right]$
		\\
		$L^* \geq 0$  & minimum cumulative loss: $L^* = \min_{a^* \in \cA} \E \left[ \sum_{t=1}^T f_t(a^*) \right] $
		\\
		$Q \geq 0$ & total quadratic variation in loss sequence: $Q = \min_{\bar{\ell}\in \re^d} \E \left[ \sum_{t=1}^T \left\| \ell_t - \bar{\ell} \right\|_2^2 \right]$
		\\
		$P \geq 0$ & path-length of loss sequence: $P = \E \left[ \sum_{t=1}^{T-1} \left\| \ell_t - \ell_{t+1} \right\|_2 \right]$
		\\
		\bottomrule
	\end{tabular}
\end{table}
\begin{table}[t]%
\caption{
  Regret bounds for stochastic and adversarial linear bandits.
  Bounds depending on $L^*$ are applicable when $f_t(a) \geq 0$.
  Bounds depending on $Q$ or $P$ are applicable when $\varepsilon_t(a) = 0$.
}
  \label{table:regretbound}
    \centering
    \begin{tabular}{lcccccccccccc}
        \toprule
        Algorithm & Stochastic & Adversarial
        \\
        \midrule
        \cite{bubeck2012towards} &  & $O\left( \sqrt{d T \log (|\cA|)} \right) $
        \\
        \cite{abernethy2008competing} &  & $O\left( d \sqrt{ \theta T \log T}\right)$
        \\
        \cite{ito2021hybrid}&  & ${O} \left( d \sqrt{ \min \{ T, L^*, Q, P \} }(\log T)^2 \right)$
	\\
	\cite{lattimore2017end}
	&
	$ c^* \log T + o \left( \log T \right)$
	&
        \\
        \cite{lee2021achieving}
        &
        $O \left( c^*  \log ( T |\cA| ) \log T \right)$
        &
        $O \left(  \sqrt{ d T} \log (T |\cA| \log T) \right)$
        \\
        \textbf{[This work]}
        &
        $O \left( (\frac{d \sigma^2}{\Delta_{\min}} + 1) d \theta \log T \right)$
        &
        $O \left( d \sqrt{\theta \min\{ T, L^*, Q, P \}  \log T}  \right) $
        \\
        \bottomrule
    \end{tabular}
\end{table}%

Table~\ref{table:regretbound} provides a comparison of our regret bounds with those in existing studies.
For stochastic settings,
the tight asymptotic regret given $\cA$ and $\ell^*$ can be characterized with $c^* = c(\cA, \ell^*)$,
a definition of which can be found in,
e.g.,
the paper by \citet{lattimore2017end}.
They have provided an algorithm that achieves an asymptotically optimal regret bound of
$R_T = c^* \log T + o(\log T)$.
However,
such asymptotically optimal algorithms are not necessarily optimal in environments with small variance $\sigma^2$.
In the case of $c^* = \Omega\left((\frac{d \sigma^2}{\Delta_{\min}} + 1 ) d \theta \right)$,
the proposed algorithm provides a better regret bound.
We would also like to emphasize the fact that our stochastic regret bound includes only a single $\log T$ factor,
while the bound by the BOTW algorithm of \citet{lee2021achieving} includes a $(\log T)^2$ factor.

For adversarial environments,
\citet{ito2021hybrid} has provided an algorithm with data-dependent regret bounds that depend on $L^*$, $Q$, and $P$ simultaneously.
In this regard,
our regret bounds here have an additional factor of $\sqrt{d \theta}$ but are better in terms of the dependency w.r.t.~$\log T$.
Our regret bounds can be better than those with BOTW algorithm by \citet{lee2021achieving}
if the loss sequence satisfies $\min\{ L^*, Q, P \} = O\left(T \log (T |\cA|) / \sqrt{d \theta \log T}\right)$.

For corrupted stochastic environments,
the BOTW algorithm by \citet{lee2021achieving} achieves a regret bound of
$O \left( \frac{d  \log (T |\cA|) \log T}{\Delta_{\min}} + C \right)$,
while our bound is
$O \left( \Rs + \sqrt{ \Rs  C } \right) $,
where $\Rs$ satisfies
$\Rs \leq O \left( (\frac{\sigma^2}{\Delta_{\min}} + 1) d^2 \theta\log T \right)$.
As
$ \sqrt{ \Rs  C }
\leq \frac{1}{2}( \Rs + C) $
follows from the AM-GM inequality,
our algorithm also implies $R_T = O \left( \Rs + C \right) $,
which is superior to the bound by \citet{lee2021achieving} when
$\sigma^2 + \Delta_{\min} = O \left( \frac{\log (T |\cA| )}{d \theta} \right)$.
We would like to stress here that the impact of corruption on the performance of our algorithm
is only of a square-root factor in $C$,
while algorithms in existing studies \citep{lee2021achieving,li2019stochastic,bogunovic2021stochastic} include at least a linear factor in $C$.
Comparison of such results w.r.t.~corrupted settings,
however,
requires particular care,
as there are differences in the details of problem settings.
\begin{remark}
	\upshape
	In this paper,
	regret is defined in terms of loss \textit{including} corruption,
	while existing studies define regret in terms of loss \textit{without} corruption.
	As the difference between these two notions of regret is at most $O(C)$,
	our algorithm enjoys $O \left( \Rs + C \right) $-bound even in terms of the latter definition of regret.
	Such a difference in models has been discussed by \citet{gupta2019better}.
\end{remark}

The main innovations for achieving high-level adaptability (i.e., the BOTW property) are with regard to the sampling method for actions.
In a previous study by \citet{abernethy2008efficient},
they compute a point $x_t$ in the convex hull $\conv (\cA)$ of the action set $\cA$ using a follow-the-regularized-leader (FTRL) approach,
and they then pick $a_t$ from the \textit{Dikin ellipsoid} $W_1(x_t) \subseteq \conv(\cA)$ that is defined from the self-concordant barrier for $\conv(\cA)$.
Here,
the action $a_t$ must be sampled so that its expectation matches $x_t$.
In addition,
the larger the variance of $a_t$,
the better estimator $\hat{\ell}_t$ for $\ell_t$ that we can construct,
i.e.,
the smaller variance of $\hat{\ell}_t$.
In this paper,
in order to improve the variance of the loss estimator,
we introduce a new technique that we refer to as \textit{scaled-up sampling}
(see Figure~\ref{fig:sampling}).
In this approach,
we construct a scaled-up set $W' \subseteq \conv(\cA)$ of the Dikin ellipsoid $W_1(x_t)$ with a reference point $z_t \in \cA$,
for which we let $\alpha_t \geq 1$ denote the scaling factor.
Rather than sampling from $W_1(x_t)$ as is done in the previous study,
we pick $a_t$ from $W'$ with probability $\alpha_t^{-1}$,
and otherwise set $a_t = z_t$ (the expectation of $a_t$ then matches $x_t$ as well).
Consequently,
the variance of $a_t$ becomes $\alpha_t$ times larger and
the variance of the loss-vector estimator becomes $\alpha_t^{-1}$ times smaller,
which contributes to the improvement of the regret upper bound.
In stochastic environments in particular,
intuitively,
$x_t$ approaches an extreme point (a truly optimal solution),
allowing for a smaller $W_1(x_t)$ and a larger value of $\alpha_t$,
which leads to a significant improvement in regret.

\begin{figure}
	\centering
	\includegraphics[scale=0.6]{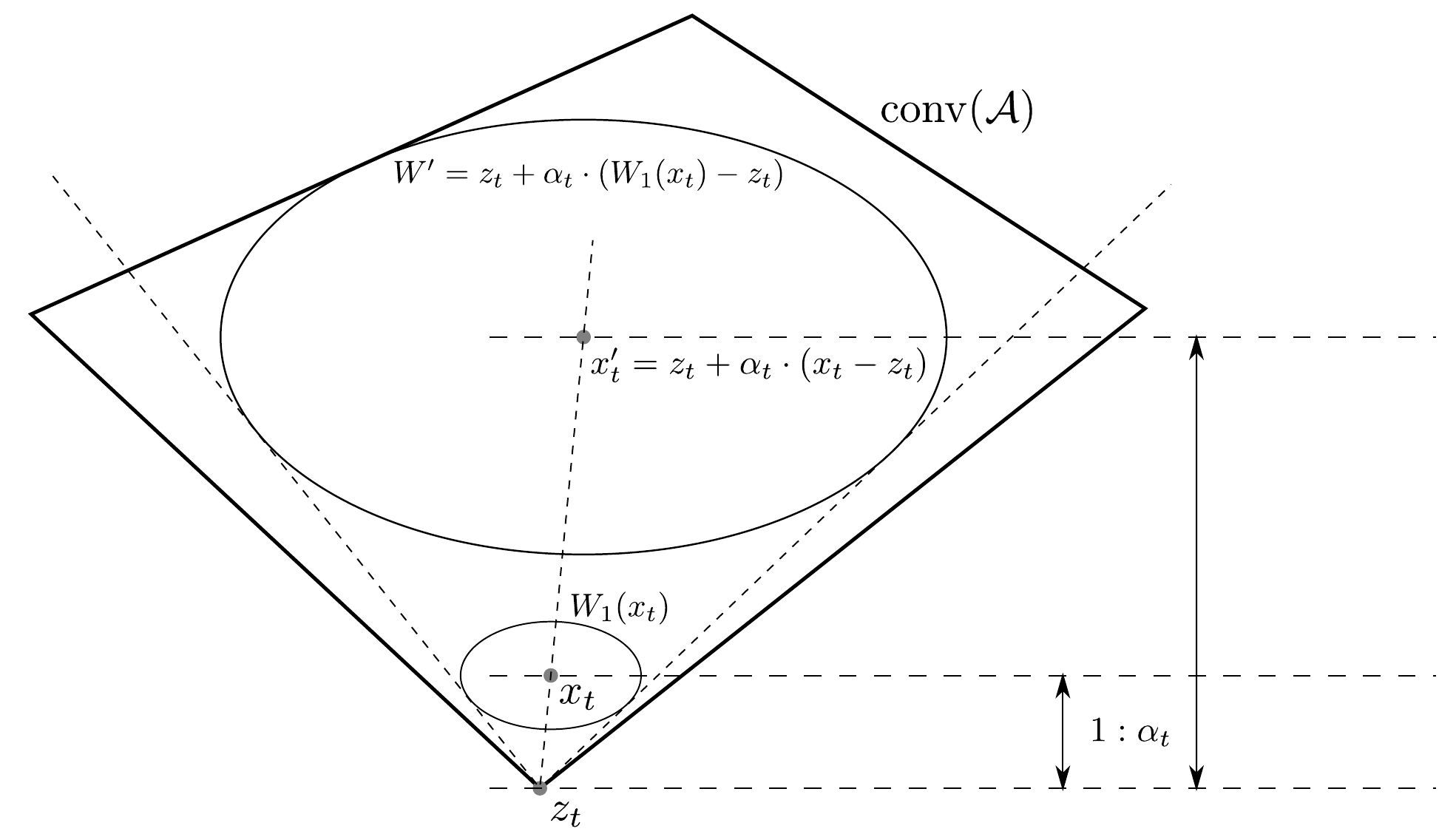}
	\caption{Illustration of scaled-up sampling.}
	\label{fig:sampling}
\end{figure}

In proving the high-level adaptability,
we use the self-bounding technique \citep{zimmert2021tsallis,wei2018more}.
We first show that the proposed algorithm,
an FTRL method with scaled-up sampling and an adaptive learning rate,
achieves a regret bound of $R_T = O\left( d \sqrt{ \theta \sum_{t=1}^T \alpha_t^{-1} \log T}\right)$.
We further show that $\alpha_t^{-1} = O(\Delta(x_t) / \Delta_{\min}) $ holds in any stochastic environment,
where $\Delta(x_t)$ denotes the round-wise regret caused by choosing $x_t$.
Combining these two facts,
we can obtain $R_T = O\left( d \sqrt{\theta \sum_{t=1}^T \Delta(x_t) \Delta_{\min}^{-1} \log T } \right)
= O\left( d \sqrt{\theta R_T \Delta_{\min}^{-1} \log T } \right)$,
which immediately leads to $R_T = O(d^2 \theta \Delta_{\min}^{-1} \log T)$ in stochastic environments.
As has been done in previous analyses using the self-bounding technique,
we can prove improved regret bounds for \textit{the stochastically constrained adversarial regime} \citep{zimmert2021tsallis} as well,
which includes corrupted stochastic environments.

To achieve low-level adaptability (i.e., data-dependent bounds in adversarial environments and variance-adaptive bounds in stochastic environments),
we employ the framework of \textit{optimistic online learning}~\citep{rakhlin2013online}.
This framework incorporates \textit{optimistic prediction} $m_t$ for $\ell_t$ into online learning algorithms,
thereby providing regret bounds depending on $( \linner \ell_t - m_t , a_t \rinner )^2$
rather than $(\linner \ell_t, a_t \rinner )^2$.
The proposed algorithm determines $m_t$ by means of the technique of \textit{tracking the best linear predictor},
which leads the hybrid data-dependent bounds and variance-adaptive bounds.
Similar approaches can be found in \citep{ito2021hybrid,ito2022adversarially}.

\subsection{Limitation of this work and future work}

We should note the issue of computational complexity w.r.t.~the proposed algorithm.
In the proof of $O(\log T)$-regret bounds for (corrupted) stochastic environments,
we need the assumption that the reference point $z_t$,
illustrated in Figure~\ref{fig:sampling},
is chosen so that the scaling factor $\alpha_t$ is (approximately) maximized.
We have not,
however,
found an efficient method for computing such a point $z_t$.
A naive method for computing such a $z_t$ requires a computational time of at least $\Omega(|\cA|)$,
which is highly expensive,
e.g.,
as in most examples of combinatorial bandits~\citep{cesa2012combinatorial}.
Resolving this issue of computational complexity will be important in future work.

There is still some room for improvement in terms of regret bounds as well.
As can be seen from Example 4 by \citet{lattimore2017end},
the gap between $c^*$ and $d/\Delta_{\min}$ can be arbitrarily large,
which implies that our stochastic regret bound is much larger than the lower bound in the worst case.
We also note that our regret bounds only hold in expectation
while regret guarantees by \citet{lee2021achieving} hold with high probability.
If we pursue high probability bounds as well,
we cannot avoid an extra $O(\log T)$ factor,
as discussed in their Appendix D,

\section{Preliminary}
\subsection{Problem setup}
\label{sec:setup}
This section introduces the setup of the linear bandit problems dealt with in this paper.
Before a game starts,
the player is given the \textit{time horizon} $T$ and
an \textit{action set} $\cA \subseteq \re^d$,
a closed and bounded set of $d$-dimensional vectors.
Without loss of generality,
we assume that $\cA$ is not included in any proper affine subspace of $\re^d$.
We also assume that all points in $\cA$ have $L_2$ norm of at most $1$,
i.e.,
$\cA \subseteq B_{2}^d(1)$,
where $B_2^{d}(r)$ denotes an $L_2$ ball of the radius $r$:
$B_2^{d}(r) = \{ x \in \re^d \mid \| x \|_2 \leq r \}$.
In each round $t = 1, 2, \ldots, T$,
the environment determines a \textit{loss function} $f_t: \cA \rightarrow [-1, 1]$,
and the player then chooses an \textit{action} $a_t \in \cA$ without knowing $f_t$.
After that,
the player observes the incurred loss $f_t(a_t)$.
The loss function $f_t$ can be chosen depending on the actions $(a_s)_{s=1}^{t-1}$ selected so far.
We assume that the conditional expectation of $f_t$ given $(a_s)_{s=1}^{t-1}$ is an affine function,
i.e.,
there exists $\ell_t \in \re^d$,
which is referred to as a \textit{loss vector},
and $\xi_t \in \re$ such that $f_t$ is expressed as
\begin{align}
	f_t(a) = \linner \ell_t, a \rinner + \varepsilon_t(a),
	\quad
	\mbox{where}
	~
	\E \left[ \varepsilon_t(a) | (a_s)_{s=1}^{t-1} \right] = \xi_t
	~ \mbox{for all}~
	a \in \cA.
\end{align}
This paper also assumes that $\ell_t \in B_2^d(1)$.
By imposing further conditions on $f_t$,
we can express a variety of regimes,
as are discussed below:

\paragraph{Stochastic regime}
In a \textit{stochastic regime},
it is assumed that $f_t$ follows an unknown distribution $\cD$ for all $t \in [T]$ independently.
This assumption implies that $\ell_t$ and $\xi_t$ do not change over all rounds,
i.e.,
there exists a \textit{true loss vector} $\ell^* \in \re^d$ and $\xi^*$ such that $\ell_t = \ell^*$ and $\xi_t = \xi^*$ hold for all $t \in [T]$.
Note here that standard stochastic settings also assume that functions of $\varepsilon_t$ represent \textit{zero-mean} noise,
i.e.,
$\xi^* = 0$.
This assumption is,
however,
not necessary in the algorithm proposed in this paper.
Moreover,
the proposed algorithm does not even require the assumption that $\varepsilon_t$ follows an identical distribution,
details of which will be discussed in Section~\ref{sec:result}.

\paragraph{Adversarial regime}
In the \textit{adversarial regime},
by way of contrast to the stochastic regime,
$( \ell_t  )_{t=1}^T$ is an arbitrary sequence.
More precisely,
$\ell_t$ can be chosen in an adversarial way depending on $(a_s)_{s=1}^{t-1}$.
Though adversarial environments considered in previous studies are often free from noise,
i.e.,
$\varepsilon_t(a) = 0$ is assumed,
most algorithms work well as long as the noise follows bounded zero-mean distributions.
The proposed algorithm in this paper does not require this assumption as well.

\paragraph{Stochastic regime with adversarial corruption}
The \textit{stochastic regime with adversarial corruption}
is an intermediate regime between stochastic and adversarial regimes.
It is parametrized by a true loss vector $\ell^* \in B_2^d(1)$ and by a \textit{corruption level} $C \geq 0$.
In this regime,
the sequence of $(\ell_t)_{t=1}^T$ is subject to the constraint that $\sum_{t=1}^T \| \ell_t - \ell^* \|_2 \leq C$.
This can be interpreted as a situation in which an adversary adds a corruption of $c_t = \ell_t - \ell^*$ to the loss function defined by $\ell^*$
and the magnitude of $c_t$ sums up to $C$ at most,
i.e.,
$\sum_{t=1}^T \| c_t \|_2 \leq C$.
If we set the condition level $C$ to zero,
this regime coincides with the stochastic regime.
On the other hand,
if $C = \Omega(T)$,
then the regime is adversarial
as there are no constraints on $\ell_t$ except for $\| \ell_t \|_2 \leq 1$.

\subsection{Follow the regularized leader}
In the proposed algorithm,
we use the framework of (optimistic) \textit{follow-the-regularized-leader (FTRL)} methods.
In this framework,
we choose a point $x_t$ in a closed convex set $\cX \subseteq \re^d$ by solving the following optimization problem:
\begin{align}
	\label{eq:defOFTRL}
	x_t \in \argmin_{x \in \cX}
	\left\{
		\linner m_t + \sum_{s=1}^{t-1} \hat{\ell}_s ,  x \rinner + \psi_{t} (x)
	\right\},
\end{align}
where $\hat{\ell}_s$ is the (estimated) loss vector,
$m_t \in \re^d$ is an \textit{optimistic prediction},
and $\psi_t(x)$ is a \textit{regularization term},
which is a differentiable convex function over $\cX$.
Note that the original FTRL framework here does not employ optimistic prediction,
i.e., the value of $m_t$ is fixed to $0$.
The technique of optimistic prediction $m_t$ has been introduced to further improve the performance of FTRL,
e.g.,
by \citet{rakhlin2013online}.

In the analysis of FTRL,
we use the \textit{Bregman divergence} $D_{\psi}$ associated with some differentiable convex function $\psi$ defined as follows:
\begin{align}
	D_{\psi} (x, y)
	=
	\psi(x) - \psi(y) - \linner \nabla \psi(y), x - y \rinner,
\end{align}
where $\nabla \psi (y)$ denotes the gradient of $\psi$ at $y$.
We can easily see that $D_{\psi}(x, y) \geq 0$ for any $x$ and $y$,
which follows from the convexity of $\psi$.
The following lemma provides an upper bound of the regret for FTRL:
\begin{lemma}
	\label{lem:OFTRL}
	We assume that $\psi_1(x) \geq 0$ and $\psi_{t+1}(x) \geq \psi_{t}(x)$ hold for all $x$ and $t$.
	If $x_t$ is given by \eqref{eq:defOFTRL},
	it holds for any $x^* \in \mathrm{int}(\cX)$ that
	\begin{align}
		\label{eq:lemOFTRL}
		\sum_{t=1}^T
		\linner
			\hat{\ell}_t,
			x_t - x^*
		\rinner
		\leq
		\sum_{t=1}^T
		\left(
			\linner
				\hat{\ell}_t - m_t,
				x_t - \tilde{x}_{t+1}
			\rinner
			- D_{\psi_t}( \tilde{x}_{t+1}, x_t )
		\right)
		+ \psi_{T+1}( x^* ),
	\end{align}
	where $\tilde{x}_t$ is defined by
		$
		\tilde{x}_t \in \argmin_{ x \in \cX }
		\left\{
			\linner
			\sum_{s=1}^{t-1} \hat{\ell}_s,
			x
			\rinner
			+ \psi_{t}( x )
		\right\}
		$.
\end{lemma}
This lemma can be shown via a standard analysis for FTRL,
e.g.,
as in Chapter 28 of \citet{lattimore2018bandit}.
We can also refer to,
e.g.,
the proof of Lemma~1 by \citet{ito2022adversarially}.

\subsection{Self-concordant barriers}
\label{sec:self-concordant}
In our proposed algorithm,
we use \textit{self-concordant barriers} to define regularization terms,
just as \citet{abernethy2008efficient} did.
Self-concordant barriers are defined as follows:
\begin{definition}
	\upshape
	A convex function $\psi : \mathrm{int} (\cX) \rightarrow \re$ of class $C^3$ is called a \textit{self-concordant function}
	if
    (i)
		$|D^3 \psi(x)[h, h, h]| \leq 2 (  D^2 \psi(x)[h, h] )^{3/2}$
	holds for any $x \in \mathrm{int}(\cA)$ and $h \in \re^d$, and
    (ii)
    $\psi(x_i)$ tends to infinity along every sequence $x_1 , x_2, \ldots \in \mathrm{int}(\cX) $ converging to a boundary point of $\mathrm{int}(\cX)$,
    where $D^k  \psi(x)[h_1, \ldots, h_k]$ denotes the value of the $k$-th differential of $\psi$ at $x$ along the directions $h_1, \ldots, h_k$.
	Let $\theta \geq 0$ be a non-negative real number.
	A self-concordant function $\psi: \mathrm{int}(\cX) \rightarrow \re$ is called a \textit{$\theta$-self-concordant barrier} for $\cX$ if $|D\psi(x)[h]| \leq \theta^{1/2}(D^2 \psi(x)[h,h])^{1/2}$ holds for any $x \in \mathrm{int}(\cX)$ and $h \in \re^d$.
\end{definition}

\begin{remark}
	\upshape
	For any convex set $\cX \subseteq \re^d$,
	there exists a $d$-self-concordant barrier for $\cX$ \citep{lee2021universal}.
	This barrier is,
	however,
	not always efficiently computable.
	On the other hand,
	for any $d$-dimensional polytope,
	we can compute an $\theta$-self-concordant barrier with $\theta = {O}(d)$ in polynomial time \citep{lee2014path,lee2019solving}.
\end{remark}

Given a self-concordant barrier $\psi: \mathrm{int}(\cX) \rightarrow \re$,
for any $x \in \mathrm{int}(\cX)$ and $h \in \re^d$,
we assume that $\nabla^2 \psi(x)$ has full rank.
Denote
\begin{align}
	\| h \|_{x, \psi} = \sqrt{ h^\top \nabla^2 \psi(x) h },
	\quad
	\| h \|_{x, \psi}^* = \sqrt{ h^\top (\nabla^2 \psi(x))^{-1} h }
\end{align}
and define the \textit{Dikin's ellipsoid} $W_{r}(x) \subseteq \re^d$ of $\psi$ centered at $x$ of the radius $r > 0$ as follows:
\begin{align}
	W_{r}(x) = \left\{ y \in \re^d \mid \| y - x \|_{x, \psi} \leq r \right\}.
\end{align}

The three lemmas below are used in the design and analysis of our proposed algorithm.
\begin{lemma}[Theorem 2.1.2 by \citet{nesterov1994interior}]
	\label{lem:ellipsoid}
	If $\psi$ is a self-concordant barrier for a closed convex set $\cX$,
	every Dikin's ellipsoid of $\psi$ of radius $1$ is contained in $\cX$,
	i.e.,
	$W_1(x) \subseteq \cX$ holds for any $x \in \mathrm{int}(\cX)$.
\end{lemma}
Let $\pi_{z, \cX} (x)$ denote the Minkowsky function of $\cX$ whose pole is at $z$:
\begin{align}
	\label{eq:defMinkowsky}
	\pi_{z, \cX} (x)
	=
	\inf \left\{ r > 0 \mid z + r^{-1}(x - z) \in \cX \right\}.
\end{align}
We have an upper bound on $\psi$ expressed with this Minkowsky function,
as follows:
\begin{lemma}[Propositoin 2.3.2 by \citet{nesterov1994interior}]
	\label{lem:boundpsi}
	If $\psi$ is a $\theta$-self-concordant barrier for $\cX$,
	it holds for any $x$ and $y$ in $\mathrm{int}(\cX)$ that
		$
		\psi(x) \leq \psi(y) + \theta \log \frac{1}{1 - \pi_{y,\cX}(x)}
		$.
\end{lemma}
If we use a self-concordant barrier $\psi$,
we can use the following lemma to bound the \textit{stability term}
$
	\left(
		\linner
			\hat{\ell}_t - m_t,
			x_t - x'_{t+1}
		\rinner
		- D_{\psi_t}( x'_{t+1}, x_t )
	\right)
$ in Lemma~\ref{lem:OFTRL}.

\begin{lemma}
	\label{lem:boundstability}
	Let $\psi$ be a self-concordant function on $\mathcal{X}$ and $x, y \in \mathrm{int}(\mathcal{X})$.
	Let $\beta > 0$ and $\ell \in \mathbb{R}^d$.
	Suppose that $\|\ell\|_{x, \psi}^* \le \beta/3$.
	We then have
		$
		\langle \ell, x-y \rangle - \beta D_{\psi}(y, x) \le \frac{2}{\beta}\|\ell\|_{x, \psi}^{*2}.
		$
\end{lemma}

\section{Algorithm}
Let $\cX$ be the convex hull of $\cA$
and $\psi$ be a $\theta$-self-concordant barrier for $\cX$.
In the proposed algorithm,
we compute $x_t$ by solving the optimization problem \eqref{eq:defOFTRL}
with $\psi_t(x) = \beta_{t} \psi(x)$,
where $\beta_t$ is a \textit{learning rate parameter} satisfying
$6d \leq \beta_1 \leq \beta_2 \leq \cdots$.
The manner of computing $a_t$, $\hat{\ell}_t$, $m_t$, and $\beta_t$ will be presented below.

\paragraph{Action $a_t$ and unbiased estimator $\hat{\ell}_t$ for loss vector}
After computing $x_t$,
we choose the action $a_t \in \cA$ so that
$\E [a_t | x_t] = x_t$.
Let $\{ e_1, \ldots, e_d  \}$ and $\{ \lambda_1, \ldots, \lambda_d \}$ be the set of eigenvectors and eigenvalues of $\nabla^2 \psi(x_t)$.
Define $\cE_t := \{ x_t + \lambda_i^{-1/2} e_i \mid i \in [d] \} \cup \{ x_t - \lambda_i^{-1/2} e_i \mid i \in [d] \} $.
Note that here $\cE_t \subseteq \cX$ holds since $\cE_t \subseteq W_1(x_t)$ follows from the definition of $\cE_t$
and since $W_1(x) \subseteq \cX$ follows from Lemma~\ref{lem:ellipsoid}.
In the algorithm by \citet{abernethy2008efficient},
the action $a_t$ is chosen from $\cE_t$ uniformly at random.
Unlike this existing method,
our proposed algorithm chooses an action from a set $\cE'_t$ scaled up from $\cE_t$ with a reference point $z_t \in \cA$,
or chooses $a_t = z_t$ with some probability.
More precisely,
after computing $\cE_t$ and choosing a point $z_t \in \cA$,
we set $\cE'_t$ by
\begin{align}
	\cE'_t =  \left\{ z_t + \alpha_t (x - z_t) \mid x \in \cE_t  \right\},
\end{align}
where $\alpha_t \geq 1$ is defined as the largest real number such that $\cE'_t$ is included in $\cX$.
How to choose $z_t$ is discussed in the next pragraph.
If we denote $r_t = \alpha_t^{-1} \in (0, 1]$,
we can express $r_t$ as follows:
\begin{align}
	\label{eq:defrt}
	r_t = \inf \left\{ r > 0 \mid z_t + r_t^{-1} ( x - z_t ) \in \cX \quad (x \in \cE_t) \right\}
	= \max_{x \in \cE_t} \pi_{z_t, \cX}(x) ,
\end{align}
where $\pi$ is the Minkowsky function defined by \eqref{eq:defMinkowsky}.
We choose $z_t \in \cA$ so that the value of $r_t$ is as small as possible.
Let $x'_t$ denote the center of $\cE'_t$,
i.e.,
define $x'_t = z_t + r_t^{-1} (x_t - z_t)$.
We then set $b_t = 1$ with probability $r_t$
and $b_t = 0$ with probability $1 - r_t$.
If $b_t=0$,
we choose $a'_t = z_t$.
If $b_t=1$,
we choose $a'_t$ from $\cE'_t$ uniformly at random.
In other words,
we pick
$i_t$ uniformly at random from $[d]$ and $\epsilon_t = \pm 1$ with probability $1/2$,
and set
$a'_t = z_t + r_t^{-1} ( x_t + \epsilon_t \lambda_{i_t}^{-1/2} e_{i_t} - z_t) $.
We then output $a_t \in \cA$ so that its expectation coincides with $a'_t \in \cX = \conv (\cA)$.
After obtaining feedback of $f_t(a_t)$,
we define
$\hat{\ell}_t$ by
\begin{align}
	\hat{\ell}_t
	=
	m_t +
	d b_t \epsilon_t \lambda_{i_t}^{1/2}
	(
		f_t(a_t) - \linner m_t, a_t \rinner
	)
	e_{i_t} .
	\label{eq:defellhat}
\end{align}
We can show that
the conditional expectation of $a_t$ is equal to $x_t$ and that $\hat{\ell}_t$
is an unbiased estimator of $\ell_t$,
i.e.,
we have
$
\E\left[a_t|x_t\right] = x_t
$
and
$
\E[\hat{\ell}_t|x_t] = \ell_t$,
proofs of which are given in Section~\ref{sec:unbiased} in the appendix.
We note that,
thanks to the scaled-up sampling,
the mean square of $\hat{\ell}_t - m_t$ is improved by a factor of $1/\alpha_t$,
which plays a central role in our proof of BOTW regret bounds.

\paragraph{Reference point $z_t$}
We will see that the smaller value of $r_t$ is,
the smaller variance of $\hat{\ell}_t - m_t$ is,
resulting in an improvement in regret.
To take maximum advantage of this effect,
we choose $z_t$ so that $r_t$ is as small as possible.
More precisely,
for a constant $\kappa \geq 1$,
we assume that $z_t$ satisfies
\begin{align}
	\label{eq:aspzt}
	r_t
	=
	\max_{x \in \cE_t} \pi_{z_t, \cX}(x)
	\leq
	\kappa
	\cdot
	\min_{z \in \cA}
	\max_{x \in \cE_t} \pi_{z, \cX}(x)
\end{align}
for all $t \in [T]$.
This assumption is used in our proof of $O(\log T)$-regret in stochastic environments.

\paragraph{Learning rate parameter $\beta_t$}
In the regret analysis in Section~\ref{sec:analysis},
we will show that
the regret for the proposed algorithm is bounded as
$R_T = O\left( \E \left[ d^2 \sum_{t=1}^T \frac{g_{t}(m_t)}{\beta_t} +  \beta_{T+1} \theta \log T \right]
\right)$,
where $g_t(m)$ is defined as
\begin{align}
	\label{eq:defgt}
	g_t(m) = b_t \cdot \left( \linner a_t, m \rinner - f_t(a_t) \right)^2 .
\end{align}
Intuitively,
$g_t(m_t) / \beta_t$ comes from the part of $\linner \hat{\ell}_t - m_t, x_t - \tilde{x}_{t+1} \rinner - D_{\psi_t} (\tilde{x}_{t+1}, x_t)$ in \eqref{eq:lemOFTRL},
which is called \textit{stability terms},
and $\beta_{T+1} \log T$ comes from the part of $\psi_{T+1} (x^*)$,
called \textit{penalty terms}.
To balance stability and penalty terms,
we set $\beta_t$ by
\begin{align}
	\label{eq:defbetat}
	\beta_t =
	6 d
	+
	2 d \sqrt{ \frac{\sum_{s=1}^{t-1} g_s(m_s)  }{\theta \log T} },
\end{align}
which leads to
$R_T = O\left( d \E \left[ \sqrt{ \theta \log T \cdot \sum_{t=1}^T  g_t(m_t)  } \right] +  d \log T \right)$.

\paragraph{Optimistic prediction $m_t$}
To minimize the part of $\sum_{t=1}^T g_t(m_t)$,
we choose $m_t$ by using online projected gradient descent for $g_t$.
We set $m_1 = 0$ and update $m_t$ as follows:
\begin{align}
	\label{eq:updatemt}
	m'_{t+1} = m_{t} - \eta b_t \cdot ( \linner a_t, m_t \rinner - f_{t} (a_t) ) a_t,
	\quad
	m_{t+1} = \min \left\{ 1, \frac{1}{ \| m'_{t+1} \|_2 } \right\}  m'_{t+1} ,
\end{align}
where $\eta \in (0, 1/4)$ is the learning rate parameter for updating $m_t$.

The proposed algorithm can be summarized as Algorithm~\ref{alg:FTRL} in Section~\ref{sec:pseudocode} in the appendix.

\paragraph{Computational complexity}
The procedure in each round can be performed in polynomial time in $d$,
except for the computation of $z_t$.
Indeed,
given a self-concordant barrier for $\cX$,
we can solve an arbitrary linear optimization problem over $\cX$ (and thus also over $\cA$),
with the aid of,
e.g.,
interior point methods \citep{nesterov1994interior}.
This implies that convex optimization problems \eqref{eq:defOFTRL} can be solved in polynomial time as well.
Futher,
for any $a'_t \in \cX$,
we can find an expression of convex combination of points in $\cA$ in polynomial time
\citep[Corollary 11.4]{mirrokni2017tight,schrijver1998theory},
which means that we can randomly choose $a_t \in \cA$ so that $\E[a_t | a'_t] = a'_t$.
As for the calculation of $z_t$ satisfying \eqref{eq:aspzt},
it is not clear if there is a computationally efficient way at this point.
Because we can compute the value of $\max_{x \in \cE_t} \pi_{z, \cX}(x)$ for any $z \in \cA$
in polynomial time in $d$,
we can find $z_t$ minimizing this value in $O(\mathrm{poly}(d) |\cA|)$ time,
which can be exponential in $d$.

\section{Analysis}
\label{sec:analysis}

\subsection{Regret bounds for the proposed algorithm}
\label{sec:result}
\begin{theorem}[Regret bounds in the adversarial regime]
	\label{thm:adv}
	Let $L^*$, $Q$ and $P$ be parameters defined as in Table~\ref{table:parameters}.
	The regret for Algorithm~\ref{alg:FTRL} is bounded as
	\begin{align}
		\label{eq:adv1}
		R_T =
		O \left(
			d \sqrt{ \theta \log T 
			\left(
				\min \left\{ Q, P \right\}
				+
				\E \left[ \sum_{t=1}^T (\varepsilon_t( a_t ))^2 \right]
			\right)
			}
			+
			d \theta \log T
		\right).
	\end{align}
	Further,
	if $f_t(a) \geq 0$ for any $a \in \cA$ and $t \in [T]$,
	we have
	\begin{align}
		\label{eq:adv2}
		R_T =
		O \left(
			d \sqrt{ \theta L^* \log T   }
			+ d \theta \log T
		\right) .
	\end{align}
\end{theorem}

Note that the regret bounds in Theorem~\ref{thm:adv} are valid regardless of the choice of $z_t$.
In fact,
we can demonstrate these regret bounds even if we sample $a_t$ from $\cE_t$,
as is similarly done with the algorithm by \citet{abernethy2008efficient},
which corresponds to $r_t = 1$.
By way of contrast,
to show $O(\log T)$-regret bounds for stochastic environments,
we need the assumption of \eqref{eq:aspzt}.
Under this assumption,
we have the following regret bounds:
\begin{theorem}[Regret bounds in the corrupted stochastic regime]
	\label{thm:sto}
	Let $\ell^* \in \re^d$ and denote $C = \sum_{t=1}^T \| \ell_t - \ell^* \|_2$.
	Define $a^* \in \argmin_{a \in \cA} \linner \ell^*, a \rinner$ and
	$\Delta_{\min} = \min_{ a \in \cA \setminus \{ a^* \} } \linner \ell^*, a - a^* \rinner$.
	We have
	$
	R_T
	=
	O\left(
		d
		\sqrt{  \left( C +  \sum_{t=1}^T \sigma_{t}^2 \right) \theta \log T }
		+
		d \theta \log T
	\right)$,
	where we define
	$\sigma_{t}^2 = \max_{a \in \cA} \E [ (\varepsilon_t(a))^2 ]$.
	Further,
	if $a^* \in \argmin_{a \in \cA} \linner \ell^*, a \rinner$ exists uniquely,
	under the assumption of \eqref{eq:aspzt},
	we have
	\begin{align}
		\label{eq:thmsto}
		R_T (a^*)
		=
		O\left(
			\left( \frac{\kappa d \sigma^2}{\Delta_{\min}} + 1 \right) d\theta \log T
			+
			\sqrt{
				\left( \frac{ \kappa \sigma^2}{\Delta_{\min}}  + 1 \right) C d^2 \theta \log T
			}
		\right) ,
	\end{align}
	where $\sigma^2  = \max_{t \in [T]} \sigma_{t}^2$.
\end{theorem}

\begin{remark}
	\label{rem:sto}
	\upshape
	In standard settings of the stochastic regime,
	it is assumed that $f_t$ follows an identical distribution for different rounds
	and $\E [ \varepsilon_t(a) | (a_s)_{s=1}^{t-1} ] = 0$ for all $a$.
	Such assumptions are not,
	however,
	needed in Theorem~\ref{thm:sto}.
	In other words,
	even when $\xi_t = \E [ \varepsilon_t(a) | (a_s)_{s=1}^{t-1} ]$ is non-zero and changes depending on $t$,
	we still have the $O(\log T)$-regret bounds given in Theorem~\ref{thm:sto}.
\end{remark}

\subsection{Proof sketch}
Regret bounds in Theorems~\ref{thm:adv} and \ref{thm:sto} are derived from the following lemma:
\begin{lemma}
	\label{lem:RTgt}
	The regret for Algorithm~\ref{alg:FTRL} is bounded as follows:
	\begin{align}
		R_T
		=
		O \left(
			d
			\cdot
			\E \left[
				\sqrt{\theta \log T \sum_{t=1}^T g_t(m_t) }
			\right]
			+ d \theta \log T
		\right).
	\end{align}
\end{lemma}
In proving this lemma,
we use Lemmas~\ref{lem:OFTRL}, \ref{lem:boundpsi} and \ref{lem:boundstability}.
From Lemma~\ref{lem:boundstability},
the stability term 
$\linner \hat{\ell}_t - m_t , x_t - \tilde{x}_{t+1} \rinner - D_{\psi_t} ( \tilde{x}_{t+1}, x_t )$ 
in Lemma~\ref{lem:OFTRL}
is bounded by
$\frac{2}{\beta_t} \| \hat{\ell}_t - m_t \|_{x, \psi}^{*2} = \frac{2}{\beta_t} d^2 g_t(m_t)$.
From Lemma~\ref{lem:boundpsi},
we can bound the penalty term $\psi_{T+1}(x^*)$ in Lemma~\ref{lem:OFTRL} as $ \psi_{T+1}(x^*) \leq \beta_{T+1} \theta \log T$.
Combining these bounds,
we obtain
$R_T = O \left( \E \left[ d^2 \sum_{t=1}^T \frac{g_t(m_t)}{\beta_t} + \beta_{T+1} \theta \log T \right] \right)$.
From this and the definition of $\beta_t$ given in \eqref{eq:defbetat},
we have the regret bound in Lemma~\ref{lem:RTgt}.
A complete proof of this lemma is given in Section~\ref{sec:unbiased} in the appendix.

From the result of \textit{tracking linear experts} \citep{herbster2001tracking},
we obtain the following upper bound on $\sum_{t=1}^T g_t(m_t)$.
\begin{lemma}
	\label{lem:TLE}
	If $m_t$ is given by \eqref{eq:updatemt},
	it holds for any sequence $( u_t )_{t=1}^{T+1} \in (B_2^d(1))^{T+1}$ that
	\begin{align}
		\label{eq:TLE}
		\sum_{t=1}^T g_t(m_t)
		\leq
		\frac{1}{1 - 2 \eta}
		\sum_{t=1}^T g_t(u_t)
		+
		\frac{1}{\eta(1 - 2 \eta)}
		\left(
			\sqrt{2}
			\sum_{t=1}^{T}
			\| u_{t+1} - u_t \|_2
			+
			\frac{1}{2} \| u_{T+1} \|_2^2
		\right) .
	\end{align}
\end{lemma}
This lemma is a special case of Theorem 11.4 by \citet{cesa2006prediction}.

\paragraph{Proof sketch of Theorem~\ref{thm:adv}}
By substituting $u_t = \bar{\ell} \in \argmin_{\ell} \sum_{t=1}^T \| \ell_t - {\ell} \|_2^2$ for all $t$ in \eqref{eq:TLE},
we obtain
$
	\E \left[ \sum_{t=1}^T g_t (m_t) \right]
	=
	O \left( Q + \E \left[ \sum_{t=1}^T (\varepsilon_t(a_t))^2 \right] + 1 \right)
$.
Similarly,
by substituting $u_t = \ell_{t}$ for all $t$ in \eqref{eq:TLE},
we obtain
$
	\E \left[ \sum_{t=1}^T g_t (m_t) \right]
	=
	O \left( P + \E \left[ \sum_{t=1}^T (\varepsilon_t(a_t))^2 \right] + 1 \right)
$.
Combining these with Lemma~\ref{lem:RTgt},
we obtain \eqref{eq:adv1} in Theorem~\ref{thm:adv}.
Further,
if $f_t(a) \geq 0$,
by substituting $u_t=0$,
we obtain
$
	\E \left[ \sum_{t=1}^T g_t (m_t) \right]
	=
	O \left( L^* + R_T + 1 \right)
$,
which leads to a regret bound of
$
R_T = O \left( d \sqrt{\theta \log T \left( L^* + R_T \right)} + d \theta \log T \right)
$.
This implies that \eqref{eq:adv2} in Theorem~\ref{thm:adv} holds.

\paragraph{Proof sketch of Theorem~\ref{thm:sto}}
By setting $u_t = \ell^*$ for all $t$ in \eqref{eq:TLE},
we obtain
$
	\E \left[
		\sum_{t=1}^T g_t (m_t)
	\right]
	=
	O \left(
		\E \left[ C + \sum_{t=1}^T r_t \sigma_t^2 + 1 \right]
	\right)
$.
As we have $r_t \leq 1$,
from this bound and Lemma~\ref{lem:RTgt},
we have
$
R_T = O \left(
	d \sqrt{ (C + \sum_{t=1}^T \sigma_t^2 ) \theta \log T} + d \theta \log T
\right)
$.
We also have the following regret bound:
\begin{align}
	\label{eq:boundRTs0}
	R_T = O \left(
		d \sqrt{ \left(C + \sigma^2 \E \left[ \sum_{t=1}^T r_t \right] \right) \theta \log T} + d \theta \log T
	\right).
\end{align}
From the assumption of \eqref{eq:aspzt},
$r_t$ is bounded as
\begin{align}
	\label{eq:boundrt}
	r_t
	\leq
	\kappa \cdot \min_{z \in \cA} \left\{  \max_{x \in \cE_t} \pi_{z, \cX}(x) \right\}
	\leq
	\kappa \cdot \min_{z \in \cA} \left\{ \max_{x \in W_1( x_t )} \pi_{z, \cX}(x) \right\},
\end{align}
where the second inequality follows from $\cE_t \subseteq W_1(x_t)$.
The following lemma provides an upper bound on the right-hand side of this:
\begin{lemma}
	\label{lem:boundgamma}
	Suppose $a^* \in \argmin_{a \in \cA} \linner \ell^*, a \rinner$ uniquely exists.
	It holds for any $y \in \mathrm{int}(\cX)$ that
	\begin{align}
		\max_{x \in W_1( y )} \pi_{a^*, \cX}(x) 
		\leq
		2 \frac{\Delta(y)}{\Delta_{\min}},
		\quad
		\mbox{where}
		\quad
		\Delta(y) = \linner \ell^*,  y - a^* \rinner ,
		\quad
		\Delta_{\min} = \min_{a \in \cA \setminus \{ a^* \}} \Delta(a).
	\end{align}
\end{lemma}
By combining this lemma with \eqref{eq:boundRTs0} and \eqref{eq:boundrt},
we obtain
a bound depending on $\sum_{t=1}^T \Delta(x_t)$ as follows:
$
R_T = O \left(
	d \sqrt{ \left(C + \frac{\kappa \sigma^2}{\Delta_{\min}} \E \left[ \sum_{t=1}^T \Delta (x_t) \right] \right) \theta \log T} + d \theta \log T
\right).
$
On the other hand,
regret is bounded from below as
$
R_T(a^*) \geq \E \left[ \sum_{t=1}^T \Delta (x_t) \right] - 2 C  
$.
By combining these two bounds on $R_T$,
we obtain
\begin{align*}
	R_T(a^*)
	&
	=
	O \left(
		d \sqrt{
			\theta \log T
			\cdot
				\left(
				C +
				\frac{\kappa \sigma^2}{\Delta_{\min}}
				(R_T + C)
				\right)
		}
		+
		d \theta \log T
	\right) 
	\\
	&
	=
	O \left(
		d \sqrt{
			\frac{ \theta \kappa \sigma^2 \log T}{\Delta_{\min}}
			R_T(a^*)
		}
		+
		d
		\sqrt{
			\left(
				\frac{\kappa \sigma^2}{\Delta_{\min}} + 1
			\right)
			C
			\theta \log T
		}
		+
		d \theta \log T
	\right) .
\end{align*}
As $X = O( \sqrt{AX} + B)$ implies $X = O(A + B)$,
we have
\begin{align*}
	R_T(a^*)
	&
	=
	O \left(
		\frac{ d^2 \theta \kappa \sigma^2 \log T}{\Delta_{\min}}
		+
		d
		\sqrt{
			\left(
				\frac{\kappa \sigma^2}{\Delta_{\min}} + 1
			\right)
			C
			\theta \log T
		}
		+
		d \theta \log T
	\right)
	\\
	&=
	O \left(
		\left(
		\frac{ d \kappa \sigma^2}{\Delta_{\min}}
		+
		1
		\right)
		d
		\theta 
		\log T
		+
		d
		\sqrt{
			\left(
				\frac{\kappa \sigma^2}{\Delta_{\min}} + 1
			\right)
			C
			\theta \log T
		}
	\right) ,
\end{align*}
which means that \eqref{eq:thmsto} holds.
A complete proof is given in Section~\ref{sec:prfsto} of the appendix.

\bibliographystyle{abbrvnat}
\bibliography{reference}

\newpage
\appendix

\section{Related Work}
\paragraph{Best-of-Both-Worlds Bandit Algorithms}
Best-of-both-worlds algorithms have been developed for various settings of multi-armed bandit (MAB) problems,
including the standard MAB problem \citep{bubeck2012best,seldin2014one,zimmert2021tsallis,ito2022adversarially,honda2023follow},
combinatorial semi-bandits \citep{zimmert2019beating,ito2021hybrid,tsuchiya2023further},
partial monitoring problems \citep{tsuchiya2023best},
episodic Markov decision processes \citep{jin2020simultaneously,jin2021best},
and linear bandits \citep{lee2021achieving}.
While most of these studies focuses only on high-level adaptability,
the algorithms by \citet{ito2022adversarially,tsuchiya2023best} for
the MAB problem and combinatorial semi-bandit problems
have low-level adaptability as well,
similarly to our proposed algorithm.
In fact,
their algorithms are best-of-three-worlds algorithms with multiple data-dependent regret bounds as well as
variance-adaptive regret bounds.
Their algorithms are also similar to ours in that it is based on the optimistic follow-the-regularizer approach
with an adaptive learning rate.
As the class of linear bandits problem includes the multi-armed bandit problem,
the results in this paper can be interpreted as an extension of their results.
Regret bounds by \citet{ito2022adversarially} are,
however,
better than ours in terms of the dependency on the dimensionality of the action set (or the number of arms)
and in that they depend on arm-wise sub-optimality gaps.

\paragraph{Adversarial Corruption}
There are several studies on the stochastic environment with adversarial corruption in the linear bandit problem~\citep{li2019stochastic,bogunovic2021stochastic,lee2021achieving} and the sibling problems such as the multi-armed bandits~\citep{lykouris2018stochastic,gupta2019better,zimmert2021tsallis,yang2020adversarial} and the linear Markov decision processes~\citep{lykouris2021corruption}.
These studies and this paper have different assumptions and regret.
This paper and \citet{lee2021achieving} assume that corruption depends only on information in the past rounds and is an affine function of the chosen action.
On the other hand, \citet{li2019stochastic,bogunovic2021stochastic} allow corruption to be any (possibly non-linear) function.
Furthermore, \citet{bogunovic2021stochastic} consider the corruption that depends on the action chosen in that round.
We also note that the definitions of the corruption level in these studies are slightly different.
While this paper includes the corruption in regret, \citet{li2019stochastic,bogunovic2021stochastic,lee2021achieving} do not.
It is known that we can convert one to the other by an additional $O(C)$-regret.
Moreover, the regret bounds in these existing studies have linear terms with respect to $C$.
Thus, our regret bound for the corrupted stochastic regime have the same dependence of the corruption as in these studies, but not vice versa.

\paragraph{Misspecified Linear Contextual Bandits}
The corrupted stochastic regime is a special case of the misspecified linear contextual bandits without knowledge of the misspecification~\citep{lattimore2020learning,foster2020adapting,pacchiano2020model,takemura2021parameter,krishnamurthy2021adapting}.\footnote{
Note that some studies assume oblivious adversary~\citep{lattimore2020learning,foster2020adapting,krishnamurthy2021adapting}, i.e., the approximation errors do not depend on the actions chosen in the past.}
This problem assumes that the expected loss functions can be approximated by a linear function.
While the approximation error can be any function of the information in the past and the current rounds in general,
the corrupted stochastic regime assumes that the approximation error is an affine function of the action chosen in the current round.
It is an open question whether the proposed algorithm can obtain a regret upper bound similar to the known regret bounds for this problem when the approximation error can be non-linear.

\section{Pseudocode of the proposed algorithm}
\label{sec:pseudocode}

\begin{algorithm}[h]
\caption{}
\label{alg:FTRL}
\begin{algorithmic}[1]
	\REQUIRE{$T$: time horizon, $d$: dimensionality of action set, $\cA \subseteq \re^d$: action set, $\psi$: self-concordant barrier over $\cX = \conv ( \cA )$,
	$\theta \geq 1$: self-concordance parameter of $\psi$, $\eta \in (0, 1/4)$: learning rate for optimistic prediction}
	\STATE Set $m_1 = 0$.
	\FOR{$t=1, 2, \ldots, T$}
	\STATE Set $\beta_t$ by \eqref{eq:defbetat} and compute $x_t$ defined by \eqref{eq:defOFTRL}.
	\STATE Let $\{ e_1, \ldots, e_d  \}$ and $\{ \lambda_1, \ldots, \lambda_d \}$ be the set of eigenvectors and eigenvalues of $\nabla^2 \psi(x_t)$ and set $\cE_t := \{ x_t + \lambda_i^{-1/2} e_i \mid i \in [d] \} \cup \{ x_t - \lambda_i^{-1/2} e_i \mid i \in [d] \} $.
	\STATE Choose $z_t \in \cA$ and set $r_t \in (0, 1]$ by \eqref{eq:defrt}.
	\STATE Set $b_t = 1$ with probability $r_t$ and set $b_t = 0$ with probability $(1-r_t)$.
	\IF{$b_t = 1$}
	\STATE Choose $i_t $ from $[d]$ uniformly at random and set $\epsilon_t \pm 1$ with probability $1/2$.
	\STATE Set $a'_t = z_t + r_t^{-1}( x_t + \epsilon_t \lambda_{i_{t}}^{-1/2} e_{i_t} - z_t )$.
	\ELSE
	\STATE Set $a'_t = z_t$.
	\ENDIF
	\STATE Output $a_t \in \cA$ so that $\E[a_t] = a'_t$ and get feedback of $f_t(a_t)$.
	\STATE Compute $\hat{\ell}_t$ defined by \eqref{eq:defellhat} and update $m_t$ by \eqref{eq:updatemt}.
	\ENDFOR
\end{algorithmic}
\end{algorithm}

\section{Proof of Lemma~\ref{lem:boundstability}}
For the convex function $\psi$ and $x \in \mathrm{dom}(\psi)$,
denote \textit{the Newton decrement} at point $x$ by $\lambda(x, \psi)$,
i.e.,
$
\lambda(x, \psi) = \|\nabla\psi(x)\|_{x,\psi}^*
$.
\begin{lemma}[Theorem 2.2.1 by \citet{nesterov1994interior}]\label{lem:sc_hessian}
    Let $\mathcal{S}$ be an open non-empty convex subset of a finite-dimensional real vector space.
    Let $\psi$ be a self-concordant function on $\mathcal{S}$ and $x \in \mathcal{S}$.
    Then,
    for each $y \in \mathcal{S}$ such that $\|x - y\|_{x, \psi} < 1$,
    we have
    \begin{align*}
        (1 - \|x - y\|_{x, \psi})^2 \nabla^2 \psi(y)
        \preceq \nabla^2 \psi(x)
        \preceq (1 - \|x - y\|_{x, \psi})^{-2} \nabla^2 \psi(y)
    \end{align*}
\end{lemma}

\begin{lemma}[(2.21) by \citet{nemirovski2004interior}]\label{lem:nd_property}
    Let $\psi$ be a self-concordant function on $\mathcal{X}$.
    If $\lambda(x, \psi) < 1$, we have
    \begin{align*}
        \|x-x^*\|_{x, \psi} \le \frac{\lambda(x, \psi)}{1 - \lambda(x, \psi)},
    \end{align*}
    where $x^* \in \argmin_{y} \psi(y)$.
\end{lemma}

\begin{lemma}\label{lem:stability_term_bound_prox}
    Let $\psi$ be a self-concordant function on $\mathcal{X}$ and $x, y \in \mathrm{int}(\mathcal{X})$.
    Suppose that $\|x - y\|_{x, \psi} \le 1/2$.
    Then, we have
    \begin{align*}
        \langle \ell, x-y \rangle - \beta D_{\psi}(y, x) \le \frac{2}{\beta}\|\ell\|_{x, \psi}^{*2}
    \end{align*}
    for all $\ell \in \mathbb{R}^d$ and $\beta > 0$.
\end{lemma}
\begin{proof}
    Using the Cauchy-Schwarz inequality and the AM-GM inequality, we have
    \begin{align*}
        \langle \ell, x-y \rangle
        \le \|\ell\|_{x, \psi}^{*} \|x - y\|_{x, \psi}
        \le \frac{2}{\beta}\|\ell\|_{x, \psi}^{*2} + \frac{\beta}{8} \|x - y\|_{x, \psi}^2.
    \end{align*}
    Thus, it is sufficient to show $D_{\psi}(y, x) \ge \frac{1}{8}\|x - y\|_{x, \psi}^2$.

    By Taylor's theorem, we have $D_{\psi}(y,x) = \frac{1}{2}\|x-y\|_{\xi,\psi}^2$ for some $\xi = x + \alpha (y - x)$ where $\alpha \in (0,1)$.
    It follows from Lemma \ref{lem:sc_hessian} that
    \begin{align*}
        \|x-y\|_{\xi,\psi}^2
        \ge (1 - \|\xi - x\|_{x,\psi})^2 \|x-y\|_{x,\psi}^2
        = (1 - \alpha\|x-y\|_{x,\psi})^2 \|x-y\|_{x,\psi}^2
        \ge \frac{1}{4} \|x-y\|_{x,\psi}^2.
    \end{align*}
\end{proof}

\paragraph{Proof of Lemma~\ref{lem:boundstability}}
    Let $f(y) = D_{\psi}(y, x) - \langle \ell, x-y \rangle / \beta$.
    Since $\psi$ is self-concordant, there exists $y^* \in \mathrm{int}(\mathcal{X})$ such that $y^* \in \argmin_{y \in \mathcal{X}} f(y)$.
    If we have $\lambda(x, f) \le 1/3$, by Lemma~\ref{lem:nd_property}, we obtain
    \begin{align*}
        \|x - y^*\|_{x,\psi}
        = \|x - y^*\|_{x,f}
        \le \frac{\lambda(x, f)}{1 - \lambda(x, f)}
        \le 1/2.
    \end{align*}
    Thus, we obtain
    \begin{align*}
        \langle \ell, x-y \rangle - \beta D_{\psi}(y, x)
        \le \langle \ell, x-y^* \rangle - \beta D_{\psi}(y^*, x)
        \le \frac{2}{\beta}\|\ell\|_{x, \psi}^{*2},
    \end{align*}
    where the first inequality holds due to $y^* \in \argmin_{y \in \mathcal{X}} f(y)$
    and the second inequality is derived from Lemma~\ref{lem:stability_term_bound_prox}.
    Hence,
    it suffices to show $\lambda(x, f) \le 1/3$.
    By the definition of $f$, we have $\nabla f(x) = \ell / \beta$.
    Thus, we obtain
    \begin{align*}
        \lambda(x, f)
        = \|\nabla f(x)\|_{x,f}^*
        = \|\nabla f(x)\|_{x,\psi}^*
        = \|\ell\|_{x,\psi}^* / \beta
        \le 1/3,
    \end{align*}
    where the inequality is obtained by the assumption.
\qed

\section{Proof of Lemma~\ref{lem:RTgt}}
\label{sec:unbiased}
We first show that $\E \left[ a_t | x_t \right] = x_t$.
The expectation of $a_t$ is
\begin{align}
	\nonumber
	\E [a_t | x_t]
	&
	=
	\E [a'_t | x_t]
	=
	\left(1 -  r_t \right)
	\E [a'_t | b_t = 0]
	+
	r_t
	\E [a'_t | b_t = 1]
	\\
	\nonumber
	&
	=
	\left(1 -  r_t \right) z_t
	+
	r_t
	\E \left[
	z_t + r_t^{-1} \left( x_t + \epsilon_t \lambda_{i_t}^{-1/2} e_{i_t} - z_t\right)
	\right]
	\\
	&
	=
	\left(1 -  r_t \right) z_t
	+
	r_t
	(
	z_t + r_t^{-1} ( x_t  - z_t)
	)
	=
	z_t + (x_t - z_t) = x_t ,
	\label{eq:unbiasedat}
\end{align}
where the forth equality follows from $\E[\epsilon_t] = 0$.

Let us next show that $\hat{\ell}_t$ defined by \eqref{eq:defellhat} is an unbiased estimator of $\ell_t$.
We have
\begin{align}
	\nonumber
	&
	\E \left[ \hat{\ell}_t - m_t | x_t \right]
	=
	r_t
	\E \left[
	d \epsilon_t \lambda_{i_t}^{1/2}
	(f_t(a_t) - \linner m_t, a_t \rinner)
	e_{i_t}
	\right]
	\\
	&
	\nonumber
	=
	r_t
	\E \left[
	d \epsilon_t \lambda_{i_t}^{1/2}
	(\linner \ell_t - m_t, a_t \rinner + \xi_t )
	e_{i_t}
	\right]
	\\
	&
	\nonumber
	=
	r_t
	\E \left[
	d \epsilon_t \lambda_{i_t}^{1/2}
	e_{i_t}
	\linner
		z_t + r_t^{-1} (x_t + \epsilon_t \lambda_{i_t}^{-1/2} e_{i_t} - z_t),
		\ell_t - m_t
	\rinner
	\right]
	\\
	&
	\nonumber
	=
	r_t
	\E \left[
	d \epsilon_t \lambda_{i_t}^{1/2}
	e_{i_t}
	\linner
		r_t^{-1} \epsilon_t \lambda_{i_t}^{-1/2} e_{i_t} ,
		\ell_t - m_t
	\rinner
	\right]
	+
	r_t
	\E \left[
	d \epsilon_t \lambda_{i_t}^{1/2}
	e_{i_t}
	\linner
		z_t + r_t^{-1} (x_t - z_t),
		\ell_t - m_t
	\rinner
	\right]
	\\
	&
	=
	d
	\E \left[
	e_{i_t} e_{i_t}^\top (\ell_t - m_t)
	\right]
	=
	\ell_t - m_t ,
	\label{eq:unbiasedellhat}
\end{align}
where we used $\epsilon_t^2 = 1$, $\E[ \epsilon_t ] = 0$,
and the fact that $\epsilon_t$ and $m_t$ are independent
in the fifth equality.

	Suppose that $\min_{x \in \cX} \psi(x) = 0$ holds without loss of generality.
	Let $x_0 \in \argmin_{x \in \cX} \psi(x)$.
	Given $a^* \in \cA$,
	define $x^*$ by
	\begin{align*}
		x^*
		= \left( 1 - \frac{1}{T} \right) a^*
		+
		\frac{1}{T} x_0
		=
		a^*
		+
		\frac{1}{T} ( x_0 - a^* )
		.
	\end{align*}
	From this, \eqref{eq:unbiasedat} and \eqref{eq:unbiasedellhat},
	we have
	\begin{align}
		R_T(a^*)
		&
		=
		\E \left[
			\sum_{t=1}^T \linner \ell_t, a_t - a^* \rinner
		\right]
		=
		\E \left[
			\sum_{t=1}^T \linner \ell_t, a_t - x^* \rinner
		\right]
		+
		\E \left[
			\sum_{t=1}^T \linner \ell_t, x^* - a^* \rinner
		\right]
		\nonumber
		\\
		&
		=
		\E \left[
			\sum_{t=1}^T \linner \ell_t, a_t - x^* \rinner
		\right]
		+
		\frac{1}{T}
		\E \left[
			\sum_{t=1}^T \linner \ell_t, x_0 - a^* \rinner
		\right]
		\leq
		\E \left[
			\sum_{t=1}^T \linner \ell_t, a_t - x^* \rinner
		\right]
		+ 1
		\nonumber
		\\
		&
		=
		\E \left[
			\sum_{t=1}^T \linner \ell_t, x_t - x^* \rinner
		\right]
		+ 1
		=
		\E \left[
			\sum_{t=1}^T \linner \hat{\ell}_t, x_t - x^* \rinner
		\right]
		+ 1
		\label{eq:boundRT0}
	\end{align}
	Then,
	as we have $x_0 + ( 1 - 1/T )^{-1} (x^* - x_0) = x_0 + (a^* - x_0) = a^* \in \cA$,
	we have $\pi_{x_0}(x^*) \leq 1 - 1 / T$.
	Hence,
	from Lemma~\ref{lem:boundpsi},
	we have
	\begin{align*}
		\psi(x^*)
		=
		\psi(x^*) - \psi(x_0)
		\leq
		\theta \log \left( \frac{1}{1 - \pi_{x_0, \cX}(x^*) } \right)
		\leq
		\theta \log \left( \frac{1}{1 - (1 - 1/T) } \right)
		=
		\theta \log T.
	\end{align*}
	From this,
	\eqref{eq:boundRT0} and Lemma \ref{lem:OFTRL},
	we have
	\begin{align}
		R_T(a^*)
		\leq
		\E
		\left[
		\sum_{t=1}^T
		\left(
			\linner
				\hat{\ell}_t - m_t,
				x_t - x'_{t+1}
			\rinner
			- \beta_t D( x'_{t+1}, x_t )
		\right)
		+
		\beta_{T+1}
		\theta \log T
		\right]
		+
		1 .
		\label{eq:boundRT1}
	\end{align}
	The part of $
			\linner
				\hat{\ell}_t - m_t,
				x_t - x'_{t+1}
			\rinner
			- \beta_t D( x'_{t+1}, x_t )
	$ can bounded by using Lemma~\ref{lem:boundstability}.
	From the definition \eqref{eq:defellhat},
	we have
	\begin{align}
		\nonumber
		\| \hat{\ell}_t - m_t \|_{x_t, \psi}^{*2}
		&
		=
		(\hat{\ell}_t - m_t) ^\top (\nabla^2 \psi(x_t))^{-1} (\hat{\ell}_t - m_t)
		\\
		&
		\nonumber
		=
		b_t d^2 ( f_t(a_t) - \linner m_t, a_t \rinner )^2
		\lambda_{i_t}
		e_{i_t}^\top (\nabla^{2} \psi(x_t))^{-1} e_{i_t}
		\\
		&
		=
		b_t d^2 ( f_t(a_t) - \linner m_t, a_t \rinner )^2
		=
		d^2 g_t(m_t)
		\leq
		4 b_t d^2
		\leq
		4
		d^2.
		\label{eq:boundstability2}
	\end{align}
	Hence,
	if $\beta_t \geq 6 d$,
	we have $  \| \hat{\ell}_t - m_t \|_{x_t, \psi}^{*} \leq \beta_t/ 3 $,
	and, consequently,
	we can apply Lemma~\ref{lem:boundstability} to bound the stability term as follows:
	\begin{align*}
		\linner
			\hat{\ell}_t - m_t,
			x_t - x'_{t+1}
		\rinner
		- \beta_t D( x'_{t+1}, x_t )
		\leq
		\frac{2}{\beta_t}
		\| \hat{\ell}_t - m_t \|_{\psi, x_t}^{*2}
		=
		\frac{2 d^2 g_t(m_t)}{\beta_t},
	\end{align*}
	where $g_t(m)$ is defined in \eqref{eq:defgt}.
	Then,
	from this and \eqref{eq:boundRT1},
	we have
	\begin{align}
		R_T(a^*)
		\leq
		\E \left[
			2
			\sum_{t=1}^T \frac{d^2 g_t(m_t)}{\beta_t}
			+
			\beta_{T+1} \theta \log T
		\right]
		+ 1 .
		\label{eq:stabpenal}
	\end{align}
	If $\beta_t$ is given by \eqref{eq:defbetat},
	we then have
	\begin{align*}
		\frac{d^2 g_t(m_t)}{\beta_t}
		&
		=
		d
		\frac{ \sqrt{\theta \log T} g_t(m_t)}{ 6 \sqrt{\theta \log T} + 2 \sqrt{\sum_{s=1}^{t-1} g_s(m_s) }}
		\leq
		d
		\frac{ \sqrt{\theta \log T} g_t(m_t)}{ \sqrt{\sum_{s=1}^{t-1} g_s(m_s) + 36 \theta \log T} + \sqrt{\sum_{s=1}^{t-1} g_s(m_s)   } }
		\\
		&
		\leq
		d
		\frac{ \sqrt{\theta \log T} g_t(m_t)}{
			 \sqrt{\sum_{s=1}^{t} g_s(m_s) } + \sqrt{\sum_{s=1}^{t-1} g_s(m_s) } 
			 }
		=
		\sqrt{
		\theta \log T
		}
		\left( 
			\sqrt{\sum_{s=1}^{t} g_s(m_s) } - \sqrt{\sum_{s=1}^{t-1} g_s(m_s) } 
		\right).
	\end{align*}
	which yields
	\begin{align}
		\sum_{t=1}^T \frac{d^2 g_t(m_t)}{\beta_t}
		\leq
		\sqrt{
		\theta \log T
		}
		\sum_{t=1}^T
		\left( 
			\sqrt{\sum_{s=1}^{t} g_s(m_s) } - \sqrt{\sum_{s=1}^{t-1} g_s(m_s) } 
		\right)
		=
		\sqrt{ \theta \log T \cdot \sum_{t=1}^{T} g_t(m_t)  } .
		\label{eq:sumstab}
	\end{align}
	We also have
	\begin{align*}
		\beta_{T+1}\theta \log T
		=
		2
		\sqrt{ \theta \log T \cdot \sum_{t=1}^{T} g_t(m_t)  } 
		+
		6d \theta \log T
	\end{align*}
	from the definition \eqref{eq:defbetat} of $\beta_t$.
	Combining this with \eqref{eq:stabpenal} and \eqref{eq:sumstab},
	we obtain
	\begin{align*}
		R_T(a^*)
		\leq
		4 d \E \left[ \sqrt{\theta \log T \cdot \sum_{t=1}^T g_t(m_t)} \right] + 6d \theta \log T + 1,
	\end{align*}
	which completes the proof.

\section{Proof of Theorem~\ref{thm:adv}}
\label{sec:proofadv}
Fix $\eta \in (0, 1/4)$ arbitrarily.
By substituting $u_t = \bar{\ell} \in \argmin_{\ell} \sum_{t=1}^T \| \ell_t - {\ell} \|_2^2$ for all $t$ in \eqref{eq:TLE},
we obtain
\begin{align*}
	\sum_{t=1}^T g_t (m_t)
	&
	=
	O \left(
		\sum_{t=1}^T g_t( \bar{\ell}_T )
		+
		1
	\right)
	=
	O \left(
		\sum_{t=1}^T b_t \left( \linner \ell_t - \bar{\ell}_T,  a_t \rinner + \varepsilon_t(a_t) \right)^2
		+
		1
	\right)
	\\
	&
	=
	O \left(
		\sum_{t=1}^T  \left( \| \ell_t - \bar{\ell}_T \|_2^2 + ( \varepsilon_t(a_t))^2 \right)
		+
		1
	\right).
\end{align*}
Similarly,
by substituting $u_t = \ell_t$,
we obtain
\begin{align*}
	\sum_{t=1}^T g_t (m_t)
	=
	O \left(
		\sum_{t=1}^T  ( \varepsilon_t(a_t))^2
		+
		\sum_{t=1}^{T-1}  \left( \| \ell_t - \ell_{t+1} \|_2 \right)
		+
		1
	\right).
\end{align*}
By combining these with Lemma~\ref{lem:RTgt} and applying Jensen's inequality,
we obtain \eqref{eq:adv1}.
Further,
if $f_t(a) \geq 0$,
by substituting $u_t = 0$,
we obtain
\begin{align*}
	\sum_{t=1}^T g_t(m_t)
	=
	O \left(
		\sum_{t=1}^T g_t(0)
		+
		1
	\right)
	=
	O \left(
		\sum_{t=1}^T b_t ( f_t(a_t) )^2
		+
		1
	\right)
	=
	O \left(
		\sum_{t=1}^T f_t(a_t)
		+
		1
	\right) .
\end{align*}
By combining this with Lemma~\ref{lem:RTgt},
we obtain
\begin{align*}
	R_T(a^*)
	&
	=
	O \left(
		d \sqrt{ \theta  \log T
		\left(
			\E \left[
			\sum_{t=1}^T f_t(a_t)
			\right]
			+
			1
		\right)
		}
	\right)
	\\
	&
	=
	O \left(
		d \sqrt{ \theta  \log T
		\left(
			R_T(a^*)
			+
			\E \left[
			\sum_{t=1}^T f_t(a^*)
			\right]
			+
			1
		\right)
		}
	\right)
\end{align*}
which implies that \eqref{eq:adv2} holds.
\qed

\section{Proof of Lemma~\ref{lem:boundgamma}}
As $\cX$ is the convex hull of $\cA' = \{ a^* \} \cup \conv ( \cA \setminus \{ a^* \} )$,
any point $y \in \cX$ can be expressed as
a convex combination of $a^*$ and a point in $\conv( \cA \setminus \{ a^* \})$,
which means that there exists $\lambda \in [0, 1]$ and
$x' \in \conv( \cA \setminus \{ a^* \} )$ such that
$x = \lambda x' + (1 - \lambda) a^*$.
For such $x$,
we have
\begin{align}
	\label{eq:boundpi0}
	\pi_{a^*, \cX} (x) \leq \lambda.
\end{align}
In fact,
we have
\begin{align*}
	a^* +  \lambda^{-1}( x - a^* )
	=
	a^* + \lambda^{-1} \lambda ( x' - a^* )
	=
	x' \in \cX ,
\end{align*}
which means that \eqref{eq:boundpi0} holds.
We further have
\begin{align*}
	\Delta(x)
	=
	\lambda \Delta(x')
	+
	(1-\lambda) \Delta (a^*)
	=
	\lambda \Delta(x')
	\geq
	\lambda \Delta_{\min},
\end{align*}
where the last inequality follows from the fact that
$x' \in \conv(\cA \setminus \{ a^* \})$ and the definition of $\Delta_{\min}$.
Combining this with \eqref{eq:boundpi0},
we obtain
\begin{align}
	\label{eq:140}
	\pi_{a^*, \cX} (x) \leq \frac{\Delta (x)}{\Delta_{\min}} .
\end{align}
We next show
\begin{align}
	\label{eq:141}
	\max_{ x \in W_1(y)} \Delta(x)
	\leq
	2 \Delta(y).
\end{align}
As $W_1(y)$ is an ellipsoid centered at $y$,
it holds that
\begin{align*}
	\linner \ell^*, y \rinner
	-
	\min_{x \in W_1(y)} \linner \ell^*, x \rinner
	=
	\max_{x \in W_1(y)} \linner \ell^*, x \rinner
	-
	\linner \ell^*, y \rinner.
\end{align*}
We hence have
\begin{align*}
	\max_{x \in W_1(y)} \Delta(x)
	&
	=
	\max_{x\in W_1(y)}
	\linner \ell^*, x \rinner
	-
	\linner \ell^*, y \rinner
	+
	\Delta(y)
	\\
	&
	=
	\linner \ell^*, y \rinner
	-
	\min_{x\in W_1(y)}
	\linner \ell^*, x \rinner
	+
	\Delta(y)
	\\
	&
	\leq
	\linner \ell^*, y \rinner
	-
	\min_{x \in \cX}
	\linner \ell^*, x \rinner
	+
	\Delta(y)
	=
	2 \Delta(y),
\end{align*}
where the inequality follows from the fact that $W_1(y) \subseteq \cX$.
Combining \eqref{eq:140} and \eqref{eq:141},
we obtain
\begin{align*}
	\max_{x \in W_1(y)}
	\pi_{a^*, \cX} (x)
	\leq
	\max_{x \in W_1(y)}
	\frac{\Delta(x)}{\Delta_{\min}}
	\leq
	2 \frac{\Delta(y)}{\Delta_{\min}} .
\end{align*}

\section{Proof of Theorem~\ref{thm:sto}}
\label{sec:prfsto}
From Lemma~\ref{lem:TLE} with $u_t = \ell^*$,
we have
\begin{align*}
	\E \left[
		\sum_{t=1}^T g_t (m_t)
	\right]
	&
	=
	O\left(
		\E \left[
			\sum_{t=1}^T g_t (\ell^*)
		\right]
		+
		1
	\right)
	\leq
	O\left(
		\E \left[
			\sum_{t=1}^T b_t ( \linner \ell_t - \ell^*, a_t \rinner + \varepsilon_t( a_t ) )^2
		\right]
		+
		1
	\right)
	\\
	&
	\leq
	O\left(
		\E \left[
			\sum_{t=1}^T
			\left( \| \ell_t - \ell^* \|_2^2 +
			r_t \sigma_{t}^2
			\right)
		\right]
		+
		1
	\right)
	\leq
	O\left(
		\E \left[
			C +
			\sum_{t=1}^T
			r_t \sigma_{t}^2
		\right]
		+
		1 
	\right) .
\end{align*}
From this and Lemma~\ref{lem:RTgt},
we have
\begin{align}
	\label{eq:prfsto1}
	R_T = O \left(
		d
		\sqrt{ \theta \log T \cdot \E \left[ C + \sum_{t=1}^T \sigma_t^2 r_t \right]}
		+
		d \theta \log T
	\right).
\end{align}
Under the assumption of \eqref{eq:aspzt},
we have
\begin{align*}
	r_t
	\leq
	\kappa \cdot \min_{z \in \cA} \left\{  \max_{x \in \cE_t} \pi_{z, \cX}(x) \right\}
	\leq
	\kappa \cdot \min_{z \in \cA} \left\{ \max_{x \in W_1( x_t )} \pi_{z, \cX}(x) \right\}
	\leq
	2
	\kappa 
	\frac{\Delta(x_t)}{\Delta_{\min}},
\end{align*}
where second inequality follows from $\cE_t \subseteq W_1( x_t )$
and the last inequality follows from Lemma~\ref{lem:boundgamma}.
From this, \eqref{eq:prfsto1}
and $\sigma^2 = \max_{t\in [T]}\sigma_t^2$,
we have
\begin{align}
	R_T
	=
	O \left(
	d \sqrt{ \theta \log T \cdot \E \left[ C + \frac{\kappa \sigma^2}{\Delta_{\min}} \sum_{t=1}^T \Delta (x_t) \right] }
	+
	d \theta \log T
	\right) .
	\label{eq:boundRTu}
\end{align}
On the other hand,
$R_T$ is bounded from below as follows:
\begin{align}
	R_T (a^*)
	&
	=
	\sum_{t=1}^T
	\E \left[ \linner \ell_t, a_t - a^* \rinner \right]
	=
	\sum_{t=1}^T
	\E \left[ \linner \ell^*, a_t - a^* \rinner
	+
	\linner \ell_t - \ell^*, a_t - a^* \rinner
	\right]
	\nonumber
	\\
	&
	\geq
	\sum_{t=1}^T
	\E \left[ \linner \ell^*, a_t - a^* \rinner
	-
	2 \| \ell_t - \ell^* \|_2
	\right]
	=
	\E \left[
		\sum_{t=1}^T
		\linner \ell^*, x_t - a^* \rinner
		-
		2 C
	\right]
	\nonumber
	\\
	&
	=
	\E \left[
		\sum_{t=1}^T
		\Delta(x_t)
		-
		2 C
	\right] .
	\label{eq:boundRTl}
\end{align}
Combining this with \eqref{eq:boundRTu},
we obtain
\begin{align*}
	R_T(a^*)
	&
	=
	O \left(
		d \sqrt{
			\theta \log T
			\cdot
				\left(
				C +
				\frac{\kappa \sigma^2}{\Delta_{\min}}
				(R_T + C)
				\right)
		}
		+
		d \theta \log T
	\right) 
	\\
	&
	=
	O \left(
		d \sqrt{
			\frac{ \theta \kappa \sigma^2 \log T}{\Delta_{\min}}
			R_T(a^*)
		}
		+
		d
		\sqrt{
			\left(
				\frac{\kappa \sigma^2}{\Delta_{\min}} + 1
			\right)
			C
			\theta \log T
		}
		+
		d \theta \log T
	\right) .
\end{align*}
As $X = O( \sqrt{AX} + B)$ implies $X = O(A + B)$,
we have
\begin{align*}
	R_T(a^*)
	&
	=
	O \left(
		\frac{ d^2 \theta \kappa \sigma^2 \log T}{\Delta_{\min}}
		+
		d
		\sqrt{
			\left(
				\frac{\kappa \sigma^2}{\Delta_{\min}} + 1
			\right)
			C
			\theta \log T
		}
		+
		d \theta \log T
	\right)
	\\
	&=
	O \left(
		\left(
		\frac{ d \kappa \sigma^2}{\Delta_{\min}}
		+
		1
		\right)
		d
		\theta 
		\log T
		+
		d
		\sqrt{
			\left(
				\frac{\kappa \sigma^2}{\Delta_{\min}} + 1
			\right)
			C
			\theta \log T
		}
	\right) .
\end{align*}

\end{document}